\documentclass[twoside]{article}

\usepackage[preprint]{aistats2027}
\usepackage[round]{natbib}

\usepackage{color,xcolor}

\usepackage{amssymb,amsmath,amsthm,amsfonts}
\usepackage{mathtools}
\usepackage{booktabs}
\usepackage{bbm}
\usepackage{enumitem}
\usepackage{comment}

\newcommand{\one}{\mathbbm{1}}
\newcommand{\E}{\operatorname{\mathbbm{E}}}
\newcommand{\Q}{\operatorname{\mathbbm{P}}}
\newcommand{\R}{\mathbb{R}}

\newcommand{\Ii}{\mathcal{I}}

\newcommand{\Prob}{\mathbb{P}}
\newcommand{\Ftilde}{\widetilde{F}}

\newcommand{\Unif}{\mathrm{Unif}}

\theoremstyle{plain}
\newtheorem{proposition}{Proposition}[section]
\theoremstyle{definition}
\newtheorem{definition}{Definition}[section]
\theoremstyle{remark}
\newtheorem{remark}{Remark}[section]

\begin{document}

\runningtitle{Outcome-Conditional Recalibration}


\twocolumn[

\aistatstitle{Towards Calibrated Probabilistic Forecasts for Events of Interest via Outcome-Conditional Recalibration}
% Local recalibration: calibrated predictions for many outcomes
% Towards reliable intervals for high impact events via local recalibration
%Beyond Marginal Calibration: Recalibrating Probabilistic Forecasts on Outcome Subsets% Calibrated Forecasts for Events of Interest via Local Recalibration
% Calibrating Probabilistic Forecasts Across Outcome Regions
% Recalibrating Probabilistic Forecasts for Events of Interest
% Outcome-Partitioned Recalibration of Probabilistic Forecasts
% Accurate Uncertainties for Critical Events via Local Recalibration
% Towards Calibrated Probabilistic Forecasts for Critical Outcomes
% Accurate Uncertainties Across the Outcome Space via Local Recalibration
% Accurate Probabilistic Forecasts for Events of Interest via Local Recalibration
%Towards calibrated predictions for many outcomes via local recalibration
% Local recalibration: Calibrated predictive uncertainty in regions of interest

\aistatsauthor{Jakob Wessel \And Sam Allen}

\aistatsaddress{Institute for Meteorology\\ University of Leipzig \And Institute of Statistics\\ Karlsruhe Institute of Technology }]

\begin{abstract}
Calibration is an essential requirement for probabilistic predictions to be useful for decision making. While state-of-the-art prediction methods often yield miscalibrated predictive distributions, several post-hoc recalibration schemes have been proposed to generate calibrated predictions. However, popular recalibration schemes can conceal miscalibration in specific regions of the outcome space. Since particular outcomes, such as extreme events, often matter most for decision making, probabilistic predictions should be calibrated when evaluation is restricted to these outcomes. Hence, in this paper, we introduce \emph{outcome-conditional recalibration}, a post-hoc method to recalibrate probabilistic predictions on user-defined regions of the outcome space. The method is simple, easy to implement, and can be applied to arbitrary predictive distributions. It works by applying the quantile recalibration approach of \cite{kuleshov_accurate_2018} to forecast conditional distributions, before rescaling these conditional distributions so that forecast event probabilities match empirical occurrence frequencies. This produces valid and continuous predictive distributions that are calibrated within each region of interest. Across regression benchmarks, we demonstrate that existing recalibration schemes do not necessarily yield calibrated predictions when interest is on particular outcomes, and that our approach improves outcome-conditional calibration relative to existing conditional and unconditional recalibration methods, while retaining competitive calibration overall. In an application to day-ahead electricity price forecasting, the approach substantially improves calibration when predicting negative prices, at negligible cost to forecast accuracy.
\end{abstract}

\section{Introduction}

% Why should probabilistic forecasts be calibrated? 

%Probabilistic predictions quantify the uncertainty in the unknown outcome variable, providing invaluable information for decision making.
Probabilistic predictions are said to be \emph{calibrated} if they accurately quantify the uncertainty in the unknown outcome variable, in the sense that predicted probabilities align statistically with corresponding outcomes. Calibration is considered a fundamental property that forecasts should satisfy if they are to be useful for decision making \citep{gneiting_probabilistic_2007}. For example, if an event is predicted to occur with a certain probability, then the prediction is calibrated if the event does indeed occur with this probability. Here, we focus on probabilistic predictions in a regression context, where predictions take the form of a cumulative distribution function. While many forms of calibration exist in this case \citep{gneiting_regression_2023,ResinEtAl2026}, the most popular notion of calibration corresponds to all prediction intervals derived from the predictive distribution having the correct unconditional coverage \citep{Dawid1984,DieboldEtAl1998}.

% How can we get calibrated forecasts?

\cite{GuoEtAl2017} demonstrate that state-of-the-art prediction methods often do not issue calibrated predictions. Many approaches have thus been proposed to recalibrate predictions for real-valued outcomes. \cite{VovkEtAl2017} introduced conformal predictive systems, which yield sets of predictive distributions that are guaranteed to contain a calibrated prediction out-of-sample \citep[see also][]{ChernozhukovEtAl2021}. %These have since been extended to satisfy stronger conditional notions of calibration \citep{BostromEtAl2021,allen_-sample_2025}.
Several other studies have proposed weakly enforcing calibration during model training by quantifying miscalibration and using this as a regularisation term within the loss function \citep{wilks_enforcing_2018,UtpalaRai2020,LaajilEtAl2026}. Alternatively, post-hoc recalibration schemes learn the miscalibration in predictive distributions, and then transform or post-process the predictions so that this miscalibration is removed \citep{kuleshov_accurate_2018,DheurBenTaieb2024}.

% Why do we want local calibration?

However, these recalibration methods typically yield predictive distributions that only satisfy a weak unconditional notion of calibration. It is well-known that unconditional calibration is generally insufficient for effective decision making, since it only requires that predicted probabilities align with true event frequencies on average, and does not ensure that individual predictions are reliable \citep{Christoffersen1998,RomanoEtAl2019,song_distribution_2019}. %That is, biases can cancel out when averaged across different prediction cases. Biases can also cancel out when predicting different outcomes; that is, 
An unconditionally calibrated predictive distribution may also not be calibrated when attention is restricted to a particular subset of the outcome space \citep{allen_tail_2025}. Standard recalibration methods may therefore yield forecasts that are miscalibrated when attention is restricted to certain outcomes. However, certain outcomes, such as extreme events, are generally more relevant for decision making, making calibrated predictions for these outcomes particularly valuable.

% We use methods to get calibrated forecasts to get local calibration

We therefore introduce a post-hoc \emph{outcome-conditional} recalibration scheme that applies the quantile recalibration method of \cite{kuleshov_accurate_2018} to the conditional forecast distribution given that the outcome falls in a particular subset of the outcome space. By applying this repeatedly to disjoint subsets that partition the real line, and rescaling so that forecast event probabilities match empirical occurrence frequencies, the resulting recalibrated predictive distributions are valid and continuous distribution functions that satisfy outcome-conditional calibration by construction. We apply the approach to benchmark regression datasets, and find that the standard recalibration scheme of \cite{kuleshov_accurate_2018} does not necessarily issue calibrated predictions for extreme events; nor do variants designed to satisfy stronger conditional notions of calibration \citep{kuleshov_calibrated_2022}. In contrast, outcome-conditional recalibration yields well-calibrated predictions for extreme events, whilst also exhibiting competitive calibration overall. In an application to energy price forecasting, outcome-conditional calibration is capable of issuing reliable predictions of negative electricity prices, whereas alternative recalibration methods are not.

\section{Preliminaries}

\subsection{Notation}\label{sec:notation}

Let $Y \in \R$ be a scalar outcome. We denote a \emph{probabilistic prediction} for $Y$ by a cumulative distribution function $F : \R \to [0,1]$. We treat $F$ as a random variable; the randomness in $F$ often comes from dependence on random covariates, which we omit from the notation for simplicity. The joint distribution of $(F, Y)$ is denoted by $\Prob$.  We assume that this joint distribution is continuous, though this can be relaxed via suitable randomisation \citep{GneitingRanjan2013}. We refer to $F(Y)$ as the forecast \emph{probability integral transform (PIT)}. The conditional distribution of $F$ on an interval $\Ii = [a, b) \subset \R$ is defined as 
\[
F^{\Ii}(x) := \frac{F(x) - F(a)}{F(b) - F(a)}, \quad \text{for $x \in [a, b)$,}
\] 
with $F^{\Ii}(x) = 0$ for $x < a$, and $F^{\Ii}(x) = 1$ for $x \ge b$. The \emph{conditional PIT (CPIT)} on $\Ii$ is defined as $F^\Ii(Y)$. 

In practice, we observe a sample of forecast--observation pairs $\{(F_i, y_i)\}_{i=1}^{N}$, which can be interpreted as realisations of the random pair $(F, Y)$. The pairs $\{(F_i, y_i)\}_{i=1}^{N}$ can either correspond to the same data used to train the forecasts, or (more commonly) to an out-of-sample calibration dataset. Conditional forecast distributions on $\Ii$ are denoted by $\{F_i^\Ii \}_{i=1}^N$.

\subsection{Related Work}

The basis of our outcome-conditional recalibration scheme is the widely used \emph{quantile recalibration} method of \cite{kuleshov_accurate_2018}. Several studies have proposed extensions or adaptations of this PIT-based recalibration. \cite{DheurBenTaieb2024} suggest to directly learn the optimal quantile recalibrated predictive distribution, rather than applying the recalibration post-hoc; that is, the recalibration is included as a part of the model, and model parameters are then learned subject to this calibration constraint. \cite{wilks_enforcing_2018} instead adapt the loss function used to train probabilistic models by adding a regularisation term that quantifies the divergence between the PIT values $\{F_i(y_i)\}_{i=1}^N$ and a standard uniform distribution, thus penalising models that produce miscalibrated predictive distributions during training \citep[see also][]{UtpalaRai2020}. \cite{LaajilEtAl2026} employ a similar approach to yield predictive distributions for multivariate outcomes that satisfy a multivariate notion of calibration \citep{gneiting_assessing_2008,AllenEtAl2024}. 

Other studies have also proposed extensions of quantile recalibration to multivariate predictive distributions. \cite{ChungEtAl2024} use the estimated predictive density function to transform multivariate predictive distributions to univariate predictive distributions, before applying quantile recalibration. This ensures that high-density regions of the multivariate predictive distribution are calibrated. \cite{DheurBenTaieb2026} similarly propose recalibrating latent variables within generative models using quantile recalibration. Feeding these recalibrated latent variables into the estimated model then yields calibrated multivariate predictions. %This is only guaranteed to satisfy calibration of the norm in the latent space, but it at least shows an example of how PIT recalibration is being extended to the multivariate case. 
\cite{KockEtAl2026} instead directly estimate the conditional distribution of the PIT values in all dimensions. Quantile recalibration is then applied along each margin, and sampling from these recalibrated distributions at observed PIT vectors yields a sample from a calibrated multivariate predictive distribution. 

This latter approach is performed conditionally on covariates, and therefore approximates a stronger notion of calibration. \cite{song_distribution_2019} and \cite{kuleshov_calibrated_2022}, among others, also propose conditional recalibration schemes whereby the recalibration is performed conditionally on the original predictive distributions, thus making the predictions calibrated conditionally on the forecasts themselves; this is often referred to as \emph{auto-calibration} \citep{Tsyplakov2011,GneitingRanjan2013} or \emph{distribution calibration} \citep{song_distribution_2019}. However, these approaches typically require making fairly strong modelling assumptions, and therefore do not satisfy any finite-sample calibration guarantees. In Section \ref{sec:exp}, we use the approach of \cite{kuleshov_calibrated_2022} as a baseline method against which outcome-conditional recalibration is compared.

However, none of these studies have considered calibration conditional on the observed outcomes. \cite{wessel_enforcing_2026} alternatively quantify tail miscalibration, and use this to regularise loss functions during model training. This is found to yield more reliable forecasts for extreme outcomes. However, this only weakly enforces tail calibration. The recalibration scheme proposed in this paper instead provides a means to enforce outcome-conditional calibration of predictive distributions post-hoc, leading to predictive distributions that satisfy theoretical in-sample calibration guarantees.

\section{Calibration}\label{sec:cal}

The calibration of a probabilistic prediction is typically assessed by checking whether the forecast probability integral transform (PIT), $F(Y)$, follows a standard uniform distribution \citep{Dawid1984,DieboldEtAl1998}. This implies that prediction intervals derived from the predictive distribution have the correct unconditional coverage at all nominal levels. 

\begin{definition}[Probabilistic calibration]
	\label{def:pit_cal_pop}
	The prediction $F$ is \emph{probabilistically calibrated} for $Y$ if $F(Y) \sim \Unif(0, 1)$. That is, $\Prob(F(Y) \le u) = u$ for all $u \in [0,1]$.
\end{definition}

Probabilistic calibration is also often referred to as \emph{quantile calibration} \citep{song_distribution_2019}, though we adopt the terminology introduced by \cite{gneiting_probabilistic_2007}. \cite{allen_tail_2025} introduced a generalised notion of probabilistic calibration that focuses on tails of predictive distributions. This definition of \emph{tail calibration} can be adapted to assess calibration when predicting outcomes in any region of the outcome space. We refer to this as \emph{outcome-conditional calibration}.

\begin{definition}[Outcome-conditional calibration]
	\label{def:tail_cal_pop}
	The prediction $F$ is \emph{outcome-conditionally calibrated} on
	$\mathcal{I} = [a,b)$ if $\Prob(F^\Ii(Y) \le u \mid Y \in \Ii) = u$ for all $u \in [0, 1]$, \emph{and} $\Prob(Y \in \Ii) = \E[F(b) - F(a)]$.
\end{definition}

The first condition in Definition \ref{def:tail_cal_pop} requires that $F^\Ii(Y) \sim \Unif[0, 1]$ when $Y \in \Ii$. That is, the conditional predictive distribution $F^\Ii$ is a probabilistically calibrated prediction for outcomes in $\Ii$. We refer to this condition as \emph{CPIT calibration}. The second condition requires that the unconditional probability that $Y \in \Ii$ is equal to the average predicted probability that $Y \in \Ii$. This assesses whether predictions are calibrated when forecasting the occurrence of outcomes in $\Ii$. We refer to this condition as \emph{occurrence calibration}.

While outcome-conditional calibration requires that the conditional predictive distribution is calibrated when forecasting $Y \in \Ii$, probabilistic calibration generally does not imply outcome-conditional calibration \citep{allen_tail_2025}. Hence, standard recalibration schemes that enforce probabilistic calibration may not yield outcome-conditionally calibrated predictive distributions. In Section \ref{sec:exp}, we demonstrate empirically that they often do not. This motivates the introduction of outcome-conditional recalibration schemes.

\begin{remark}[Conditional calibration]
	While probabilistic calibration is most commonly assessed in practice, it corresponds to a fairly weak notion of calibration that only holds on average; it does not assess whether uncertainty is accurately quantified in any single prediction. Instead, several stronger definitions of calibration have been proposed. %A probabilistic forecast is \emph{auto-calibrated} (or distribution calibrated) if $\Prob(Y \le y \mid F) = F(y)$ almost surely for all $y \in \R$ \citep{Tsyplakov2011,GneitingRanjan2013,song_distribution_2019}. 
    For example, auto-calibration (or distribution calibration) is a strong conditional notion of calibration that extends the typical notion of conditional calibration for probabilistic classifiers. If a prediction is auto-calibrated, then it is also outcome-conditionally calibrated \citep{allen_tail_2025}. Thus, one could argue that outcome-conditional recalibration is obsolete if we can achieve auto-calibration. However, while several recalibration schemes have been proposed to achieve auto-calibration \citep{song_distribution_2019,kuleshov_calibrated_2022,KockEtAl2026}, these all require making somewhat restrictive modelling assumptions. In Section \ref{sec:exp}, we demonstrate that outcome-conditional recalibration can improve both outcome-conditional and global probabilistic calibration relative to these methods.  
\end{remark}

\section{From Global to Local Recalibration}

\subsection{Quantile Recalibration}
\label{sec:global_recal}

In a regression context, the goal of \emph{post-hoc recalibration} is to learn a function $T: [0, 1] \to [0, 1]$ such that $T \circ F$ is a calibrated predictive distribution. \cite{kuleshov_accurate_2018} consider the strictly increasing function $T(u) = \Prob(F(Y) \le u)$, with inverse $T^{-1}$. If $T$ is the identity function, the prediction $F$ is probabilistically calibrated for $Y$; if not, the prediction exhibits systematic miscalibration. For any $T$ and $F$, we have 
\[
\Prob(T(F(Y)) \le u) = \Prob(F(Y) \le T^{-1}(u)) = u,
\]
for all $u \in [0, 1]$. That is, $T \circ F$ is a probabilistically calibrated forecast distribution for $Y$ by construction. 

In practice, $T$ must be estimated from the finite sample $\{(F_i, y_i)\}_{i=1}^N$. Applying the estimated function $\hat{T}$ to the forecasts yields a sequence of predictive distributions $\{\hat{T} \circ F_i\}_{i=1}^N$ that are generally calibrated in-sample, and which should exhibit good calibration out-of-sample, assuming these forecasts-observation pairs are also realisations of $(F, Y)$. Following \cite{kuleshov_accurate_2018}, we estimate $T$ using isotonic regression, though other choices also exist, such as kernel density estimation \citep{DheurBenTaieb2023}, and the empirical distribution of the PIT values $\{F_i(y_i)\}_{i=1}^N$.
 
\subsection{Outcome-conditional Recalibration}
\label{sec:tail_recal} 

An analogous approach can be used to recalibrate the predictions within a subset $\Ii$ of interest. Define $T_\Ii(u) = \Prob(F^\Ii(Y) \le u \mid Y \in \Ii)$, for $u \in [0, 1]$. %Since we assume that $F(Y)$ has a strictly increasing distribution function, it follows that $T^\Ii$ is invertible. 
The predictive distribution is CPIT calibrated on $\Ii$ if $T_\Ii$ is the identity function. In any case, we have that
\begin{align*}
&\Prob(T_\Ii(F^\Ii(Y)) \le u \mid Y \in \Ii) = \\
&\quad \Prob(F^\Ii(Y) \le T_\Ii^{-1}(u) \mid Y \in \Ii) = T_\Ii(T_\Ii^{-1}(u)) = u,
\end{align*}
for all $u \in [0, 1]$. That is, $T_\Ii \circ F^\Ii$ is always a probabilistically calibrated predictive distribution for $Y \in \Ii$. 

The predicted probability that $Y \in \Ii$ can also easily be recalibrated by rescaling $F(b) - F(a)$ by a constant so that it satisfies the second requirement of outcome-conditional calibration in Definition \ref{def:tail_cal_pop} (whilst appropriately handling the probability mass on the complement of $\Ii$). This two-step recalibration approach can therefore be used to obtain outcome-conditionally calibrated predictions on particular regions of interest. 

%$c \in (0, \infty)$, where 
%\[
%c = \frac{\Prob(Y \in \Ii)}{\E[F(b) - F(a)]}.
%\]
%By construction, $\E[c(F(b) - F(a))] = \Prob(Y \in \Ii)$, which is 

While this only recalibrates the restriction of $F$ to $\Ii$, this approach can be applied to several intervals that partition the real line, yielding a valid predictive distribution that is locally calibrated on all chosen intervals. Let $a_0 = -\infty < a_1 < \dots < a_{G-1} < a_G = +\infty$, and consider the contiguous intervals $\Ii_1 = (a_0, a_1)$, $\Ii_g = [a_{g-1}, a_g)$, for $g = 2, \dots, G$. These $G$ intervals partition the outcome space $\R$. The recalibration scheme of \cite{kuleshov_accurate_2018} can then be applied separately to each interval. 

Given the predictive distribution $F$, define the transformation
$T_{\Ii_g}(u) = \Prob(F^{\Ii_g}(Y) \le u \mid Y \in \Ii_g)$, for $u \in [0, 1]$ and $g = 1, \dots, G$, and let $c_1, \dots, c_G \in (0, \infty)$ be rescaling constants such that 
\[
\E \left[ \frac{c_j [F(a_j) - F(a_{j-1})]}{ \sum_{g=1}^G c_g [F(a_g) - F(a_{g-1})]}  \right] = \Q(Y \in \Ii_j),
\]
for $j = 1, \dots, G$. \citet[Theorem 1]{KullFlach2015} prove that such constants exist. Define $m_0 = 0$ and $m_j = c_j [F(a_j) - F(a_{j-1})]/ \sum_{g=1}^G c_g[F(a_g) - F(a_{g-1})]$ for $j = 1, \dots, G$. Then, a recalibrated version of $F$ can be defined as 
\begin{equation}\label{eq:recal_pop}
	\Ftilde(x) = \sum_{j=0}^{g-1} m_j + m_g \left[ T_{\Ii_g} \circ F^{\Ii_g}(x) \right] \quad \text{for $x \in \Ii_g$.}
\end{equation}
We term this approach \emph{outcome-conditional recalibration}. Essentially, the quantile recalibration of \cite{kuleshov_accurate_2018} is applied separately to each conditional forecast distribution $F^{\Ii_1}, \dots, F^{\Ii_G}$. These recalibrated conditional distributions are then rescaled using the weights $m_1, \dots, m_G$. These weights always sum to one, and $\tilde{F}$ is therefore a piecewise combination of the recalibrated distributions, with knots at the interval bounds. This outcome-conditionally recalibrated predictive distribution is therefore a valid and continuous distribution function that is outcome-conditionally calibrated on the intervals $\Ii_1, \dots, \Ii_G$ by construction. This is formalised in the following propositions, which are proved in Appendix \ref{app:proofs}.
%\sam{This needs amending to use the correct rescaling constants}

\begin{proposition}[Validity]
	\label{prop:proper}
	The recalibrated predictive distribution defined at \eqref{eq:recal_pop} is a valid and continuous distribution function on $\R$. In particular, $\Ftilde(-\infty) = 0$, $\Ftilde(+\infty) = 1$, and $\Ftilde$ is non-decreasing.
\end{proposition}

\begin{proposition}[Calibration]
	\label{prop:tail_cpit}
	The recalibrated predictive distribution defined at \eqref{eq:recal_pop} is outcome-conditionally calibrated on the intervals $\Ii_1, \dots, \Ii_G$.
\end{proposition}

In practice, the functions $T_{\Ii_g}$ and rescaling constants $c_g$ are estimated from the calibration data $\{(F_i, y_i)\}_{i=1}^N$. Similarly to the global recalibration of \cite{kuleshov_accurate_2018}, the functions $T_{\Ii_g}$ are estimated using isotonic regression. However, rather than being estimated using all PIT values, $T_{\Ii_g}$ is estimated using the conditional PIT values corresponding to outcomes in $\Ii^g$, that is $\{F_i^{\Ii_g}(y_i)\}$ for $\{i = 1, \dots, N : y_i \in \Ii^g\}$. This is done separately for each $g = 1, \dots, G$. The rescaling parameters $c_g$ are estimated using iterative proportional fitting (Sinkhorn's algorithm). The resulting recalibrated predictive distributions remain valid, continuous, and locally calibrated. This sample implementation is discussed in detail in Appendix \ref{app:sample}.

\begin{remark}[Choice of intervals]
    Propositions \ref{prop:proper} and \ref{prop:tail_cpit} do not make any assumptions on the intervals $\Ii_1, \dots, \Ii^G$, and outcome-conditional recalibration can therefore be applied using any partition of the outcome space. The approach can also readily be extended to use disconnected subsets of $\R$, rather than intervals. In some applications, the choice of intervals to consider can be determined from the forecasting problem at hand. For example, in the following section, we consider an application to electricity price forecasts, where we are particularly interested in predicting negative electricity prices. In other applications, it may be reasonable to use intervals that correspond to quantiles of the outcome distribution. We find that using more intervals generally yields better overall calibration, subject to sufficient data being available on which to accurately learn the recalibration functions and parameters for each interval.
\end{remark}

\begin{remark}[Global-Local recalibration]\label{remark:global_local}
    As mentioned in Section \ref{sec:cal}, probabilistic calibration generally does not imply outcome-conditional calibration, and hence the quantile recalibration approach of \cite{kuleshov_accurate_2018} does not necessary yield outcome-conditionally calibrated predictive distributions. Similarly, there is no guarantee that outcome-conditional recalibration yields predictive distributions that are probabilistically calibrated overall; this holds regardless of the choice of intervals used within the recalibration. For comparison, we additionally consider an approach that combines the two recalibration schemes by first implementing outcome-conditional recalibration, before applying quantile recalibration to these locally calibrated predictive distributions. While this does not necessarily enforce outcome-conditional calibration, the prior application of outcome-conditional recalibration may weakly enforce this in practice.  
\end{remark}

\section{Experiments}
\label{sec:exp}

We demonstrate the utility of outcome-conditional recalibration on standardised machine learning benchmark datasets from the UCI and OpenML repositories, as well as in an application to day-ahead electricity price forecasting. In both cases, the approach is compared to the global quantile recalibration approach of \cite{kuleshov_accurate_2018}, as well as a conditional variant introduced by \cite{kuleshov_calibrated_2022}. The conditional variant is designed to recalibrate predictive distributions so that they are auto-calibrated.  

\subsection{Diagnostics}\label{sec:diagnostics}

To evaluate the various predictive distributions, we quantify the degree of miscalibration using the distance of the distribution of PIT values from a standard uniform distribution. For $u_j = j/J$, for $j = 1, \dots, J$, with $J = 10$, probabilistic calibration is assessed using
\[
    \text{Cal-Err} = \sum_{j=1}^J \left( \left[\frac{1}{N} \sum_{i=1}^N \one\{F_i(y_i) \leq u_j\} \right]- u_j \right)^2.
\]
This follows \citet{kuleshov_accurate_2018} and corresponds to the squared error between the empirical distribution function of the forecast PIT values and the distribution function of a standard uniform distribution, averaged over a discrete grid. We additionally visualise this distance to uniformity graphically via a quantile-quantile plot that shows the sorted PIT values $\{F_i(y_i)\}_{i=1}^N$ plotted against reference values $1/N, 2/N, \dots, 1$. 

Outcome-conditional calibration on an interval $\Ii = [a, b)$ is similarly assessed using the same measure of calibration error applied to the empirical distribution of the conditional PIT values. That is,
\[
    \sum_{j=1}^J \left( \left[\frac{1}{N_\Ii} \sum_{\{i : y_i \in \Ii\}} \one\{F_i^\Ii(y_i) \leq u_j\}\right] - u_j \right)^2,
\]
where $N_\Ii = |\{i = 1, \dots, N: y_i \in \Ii\}|$ is the number of observations that fall in $\Ii$. We also refer to this as subset calibration error. This only assesses CPIT calibration, the first requirement in Definition \ref{def:tail_cal_pop}. The second requirement, occurrence calibration, is assessed by computing the difference between the average predicted probability that $Y \in \Ii$ and the realised occurrence frequency. That is, for each interval $\Ii$
\[
\text{Occ-Err} = \left| \frac{1}{N} \sum_{i=1}^N \left\{ \left[ F_i(b) - F_i(a) \right] - \one\{y_i \in \Ii\} \right\} \right|.
\]
For both the calibration error and the occurrence error, a value closer to zero suggests that the predictive distributions exhibit better calibration.

\subsection{UCI and OpenML experiments}\label{sec:res_bench}

\textbf{Data.} The recalibration methods are first compared on a range of datasets from the UCI and OpenML repositories with varying sizes, from $N = 768$ to $N = 45730$. A random 25\% of the data is used for testing, and the remaining 75\% to train the models and recalibration schemes, with the split stratified so that both sets have similar outcome distributions. Results are averaged over 10 random splits. For each dataset, outcome-conditional recalibration is performed using five intervals, defined using evenly spaced quantiles of the distribution of the outcomes.

\textbf{Models.} Probabilistic predictions are first obtained using Bayesian Ridge regression \citep{mackay_bayesian_1992}, as well as MC-Dropout neural networks \citep{gal_dropout_2016} with two hidden layers of 128 units each, a dropout rate of 0.5, and PReLU activation functions. %These are standard models in the probabilistic machine learning literature. 
These predictions are recalibrated using the quantile recalibration of \citet[][KFE18]{kuleshov_accurate_2018}, the covariate-dependent recalibration  of \citet[][KD22]{kuleshov_calibrated_2022}, and outcome-conditional recalibration (OC). We additionally consider the two step procedure that implements outcome-conditional recalibration followed by quantile recalibration (OC+QR). Recalibration is performed in-sample on the training data set. We also tested fitting the recalibration methods out-of-sample, reserving 30\% of the training data for recalibration (see Appendix \ref{sec:out_of_sample_uci}). 

\begin{figure}
    \centering
    \includegraphics[width=1\linewidth]{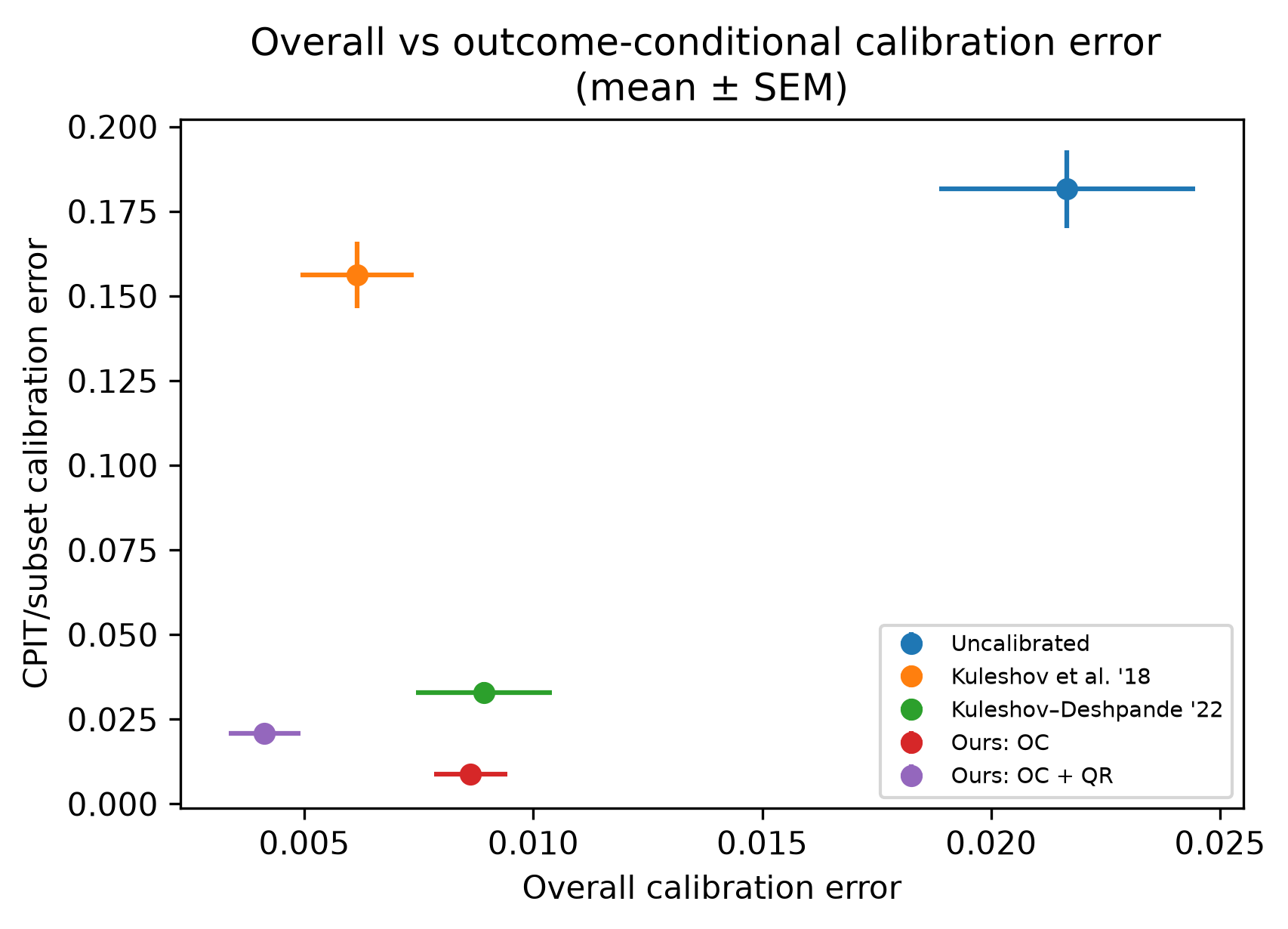}
    \caption{Outcome-conditional (CPIT/subset) vs overall calibration error averaged over all datasets and models for each of the recalibration methods. The vertical and horizontal bars indicate standard error across seeds, datasets and base methods.}
    \label{fig:summary}
\end{figure}

\textbf{Results.} Figure \ref{fig:summary} shows overall (probabilistic) and outcome-conditional (CPIT/subset) calibration error averaged over all datasets and both prediction models. Outcome-conditional calibration error is averaged across all five intervals. Table \ref{tab:subset_in_sample} in the Appendix shows outcome-conditional calibration error and occurrence calibration error, and Table \ref{tab:overall_in_sample} shows overall calibration error, for each dataset and model. 

Without recalibration, both Bayesian Ridge regression and MC-Dropout yield miscalibrated predictive distributions, which are particularly miscalibrated when attention is placed on particular outcomes. Quantile recalibration \citep{kuleshov_accurate_2018} improves overall calibration, but the resulting recalibrated predictive distributions remain severely miscalibrated within the five intervals under consideration. Covariate-conditional recalibration \citep{kuleshov_calibrated_2022} improves outcome-conditional calibration whilst marginally deteriorating overall calibration. However, as expected, outcome-conditional recalibration outperforms all other methods in terms of occurrence error and outcome-conditional calibration error; this is the case for all datasets and for both prediction models. 

Outcome-conditional recalibration leads to slightly worse overall calibration than quantile recalibration. This can be circumvented by applying quantile recalibration to the outcome-recalibrated predictive distributions; the resulting predictions exhibit the best overall calibration, without a large sacrifice in calibration when attention is restricted to particular outcomes. 

\textbf{Discussion.} The results suggest that outcome-conditional recalibration is able to successfully recalibrate predictive distributions in particular regions of the outcome space, whereas the standard quantile recalibration of \cite{kuleshov_accurate_2018} does not; this is found to hold across a range of benchmark regression datasets. The predictions are evaluated using the same partition of the outcome space on which outcome-recalibration is trained, which is the canonical approach if the partitions corresponds to subsets of interest. However, the outcome-conditional recalibrated predictive distributions also exhibit good overall calibration, and results in Appendix \ref{sec:partition_shift} demonstrate that the good outcome-conditional calibration also transfers to other partitions. The covariate-conditional recalibration of \citet{kuleshov_calibrated_2022} is competitive with outcome-conditional recalibration, though it is based on a much slower deep learning-based density fit. In contrast, outcome-recalibration is hyperparameter-free and can be applied out-of-the-box without the need for an additional validation dataset, nor complex density estimation.

\subsection{Electricity Price Forecasting}

Consider now an application to day-ahead forecasting of hourly electricity prices. We apply the recalibration methods to a state-of-the-art distributional neural network model \citep{LakshminarayananEtAl2017,RaspLerch2018}, which was proposed for electricity price forecasting by \citet{marcjasz_distributional_2023}. We focus in particular on whether the recalibration methods yield calibrated predictions for negative electricity prices, which are rare, impactful events that can occur when electricity generation exceeds electricity demand.

\textbf{Data.} We study German day-ahead electricity prices from the EPEX spot market from 2015--2020. We consider hourly electricity prices, with 24 prices per delivery day. Forecasts are issued based on a set of covariates, including day-ahead load and renewable-generation predictions; EUA carbon, API2 coal, Brent oil, and TTF gas prices; lagged electricity prices at lags of $1,2,3,7$ days; and calendar dummy variables; these are the same covariates that were used by \citet{marcjasz_distributional_2023}. The data were obtained from the ENTSO-E platform and investing.com.

\textbf{Model.} The recalibration methods are applied to a distributional deep neural network (DDNN) whose outputs are the parameters of a Johnson's $S_U$ (JSU) distribution. The JSU distribution is a flexible probability distribution that can model skewed and heavy-tailed data with support on the entire real line, which \cite{marcjasz_distributional_2023} find to be appropriate for German electricity price forecasts. The resulting predictive distribution function is $F(x)=\Phi\!\big(\gamma+\delta\,\operatorname{asinh}((x-\xi)/\lambda)\big)$, $x \in \R$, where $\Phi$ is the standard Gaussian distribution function, and $(\xi,\lambda,\delta,\gamma)$ are the estimated four parameters of the JSU distribution. Separate parameters are estimated for each hour within the same network, so that the output layer has length $4 \times 24 = 96.$

Following \cite{marcjasz_distributional_2023}, we fit a neural network with two hidden layers, which is trained using maximum likelihood estimation. An ensemble of four networks are trained with different tuned hyperparameter configurations, and the estimated JSU parameters are averaged over the four configurations. The parameters are estimated sequentially using a rolling training window: the DDNN is tuned on the first three years of data (2015--2017), before being retrained weekly on a window of the $1456$ days (roughly four years) prior to each evaluation day. The day-ahead forecasts are then evaluated over 2019 and 2020. All results are aggregated across all 24 hours of the day. 

The DDNN predictions are recalibrated using quantile recalibration, covariate-conditional recalibration, outcome-conditional recalibration (OC), as well as outcome-conditional recalibration followed by quantile recalibration (OC+QR). Like the base DDNN model, the recalibration models are also estimated adaptively using a rolling window: for each week, the recalibration parameters are estimated on a trailing $182$-day (roughly six months) window containing the DDNN's past out-of-sample predictions, and these parameters are then used to recalibrate the DDNN predictive distributions across the following week. This is designed to mirror operational practice, in which the performance of a model would be monitored sequentially over time, and the recalibration would thus also be time-adaptive.

Outcome-conditional recalibration is performed by partitioning the outcome space into a negative-price group $(-\infty,0)$ and four equal-probability quantile groups of the non-negative prices. Electricity prices are negative for $1.9\%$ of all hours between 2015--2020, which grows from $\sim\!1.44\%$ in 2015 to $\sim\!3.4\%$ in 2020. %The base model is reasonably well calibrated, but uncalibrated for negative price predictions.  

\begin{table}[t]
\vspace{-0.2cm}
\centering
\caption{Rolling (time-adaptive, 182-day) recalibration of the DDNN-JSU, German day-ahead prices, 2019--2020. CRPS in EUR/MWh; $\mathrm{cal}$ is the calibration error (lower is better); $g_0$ is the negative-price group; $\mathrm{occ}_{g_0}$ is its occurrence ratio (1 = calibrated). Best in each column in bold.\\}
\label{tab:rolling}
\setlength{\tabcolsep}{4pt}
\begin{tabular}{lcccc}
\toprule
Method & CRPS & $\mathrm{cal}_{\mathrm{all}}$ & $\mathrm{cal}_{g_0}$ & $\mathrm{occ}_{g_0}$ \\
\midrule
Raw & \textbf{2.65} & 0.0009 & 0.0611 & 0.939 \\
KFE18 & 2.67 & 0.0008 & 0.0636 & \textbf{0.973} \\
KD22 & 2.78 & 0.0006 & 0.1290 & 1.141 \\
Ours: OC & 2.70 & \textbf{0.0004} & \textbf{0.0051} & 1.050 \\
Ours: OC + QR & 2.71 & 0.0010 & 0.0082 & 1.041 \\
\bottomrule
\end{tabular}
\end{table}

\begin{figure*}[t]
    \centering
    \includegraphics[width=\linewidth]{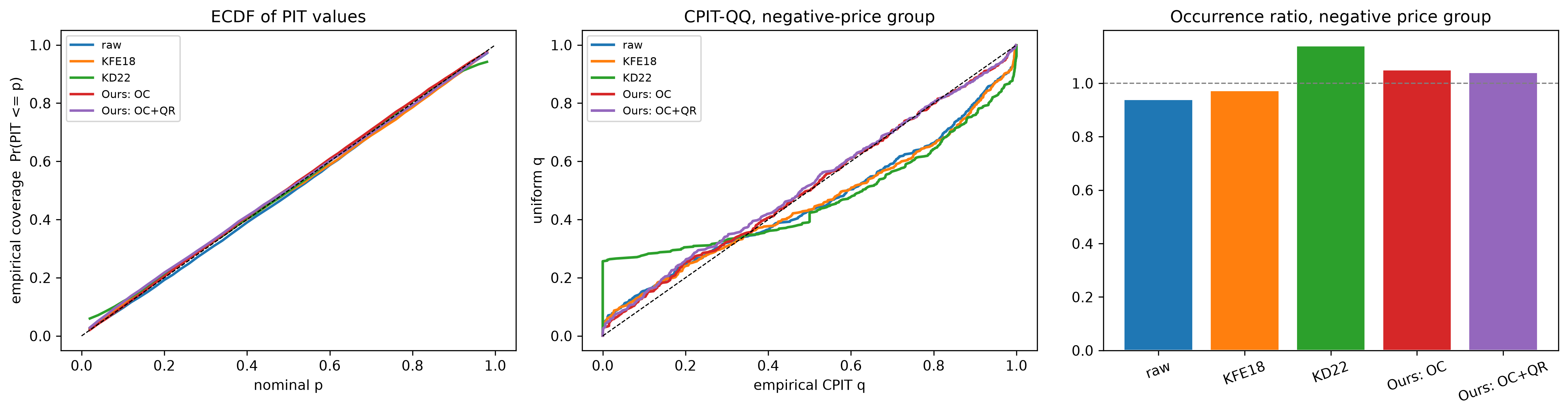}
    \caption{Rolling (182-day) recalibration of the DDNN-JSU, German day-ahead prices, 2019--2020. Left: empirical CDF of PIT values. Middle: quantile-quantile plot of CPIT values for the negative price group against reference values. Right: occurrence ratio of expected over realised events. For the left and middle panels, a calibrated forecast would lie on the 1-1 line, whereas on the right, the occurrence ratio should be close to one.}
    \label{fig:rolling}
\end{figure*}

\textbf{Results.} Table \ref{tab:rolling} displays the calibration error of all recalibration methods, as well as the raw DDNN model without recalibration. All methods yield predictive distributions that exhibit very good overall calibration; calibration error is always smaller than 0.001. The first panel of Figure \ref{fig:rolling} confirms that the empirical distribution of the PIT values is essentially uniform for all methods. 

However, the DDNN predictive distributions are not calibrated when attention is restricted to negative electricity prices, and this is not corrected by the unconditional quantile recalibration of \cite{kuleshov_accurate_2018} or the covariate-conditional recalibration of \cite{kuleshov_calibrated_2022}; this calibration error is also shown in Table \ref{tab:rolling}. In contrast, outcome-conditional recalibration yields predictive distributions that exhibit both good overall calibration, as well as good calibration when predicting negative electricity prices. 

The distribution of the conditional PIT values of all methods is shown in the second panel of Figure \ref{fig:rolling}. The CPIT values appear approximately uniform for the outcome-conditional recalibration methods, lying along the diagonal line. The shape of the distribution of CPIT values for the other methods suggests that they yield predictions that over-predict the severity of negative electricity prices, despite being calibrated on average. The final panel of Figure \ref{fig:rolling} suggests that all methods predict the occurrence of negative electricity prices with roughly the correct probability on average; the miscalibration therefore comes from the JSU distribution not accurately capturing the distribution of negative electricity prices.

Table \ref{tab:rolling} additionally contains average values of the continuous ranked probability score (CRPS), a proper scoring rule. The raw DDNN produces the most accurate predictions when assessed using the average CRPS, and all recalibration methods yield slightly higher (i.e. worse) CRPS values. This demonstrates that recalibration does not necessarily lead to improved average scores; this is the case for unconditional, covariate-conditional, and outcome-conditional recalibration. However, since the goal of probabilistic forecasting is to obtain predictions that are as sharp as possible subject to calibration, a small decrease in accuracy is often tolerable if the trade-off is improved calibration. This is especially true if calibration is improved when predicting high-impact outcomes.

\section{Discussion and conclusions}

Probabilistic predictions should be as informative as possible, subject to being calibrated. While state-of-the-art machine learning models generally do not yield calibrated predictions, several methods have been proposed to recalibrate predictions. Such recalibration methods typically only enforce a weak unconditional notion of calibration, which does not ensure that predictions are calibrated when attention is restricted to particular outcomes of interest. However, in many applications, it is essential to have calibrated probabilistic predictions for certain outcomes, since certain outcomes, such as extreme events, are often particularly important for decision making; calibrated predictions for these outcomes are therefore especially valuable.

In applications to electricity price forecasting and benchmark regression datasets, we demonstrate that state-of-the-art recalibration methods generally do not yield predictive distributions that issue calibrated forecasts for certain outcomes of interest. We therefore introduce a novel post-hoc \emph{outcome-conditional recalibration} method that ensures probabilistic predictions are calibrated when forecasting particular outcomes. The method is simple and hyperparameter-free, meaning it can be estimated almost instantaneously even for large datasets. It can be implemented in-sample on the training dataset, or (ideally) out-of-sample on an additional calibration dataset. Hence, there is large potential to implement outcome-conditional recalibration in many other application domains.

Outcome-conditional recalibration yields predictive distribution that are outcome-conditionally calibrated in-sample, on the calibration data with which they are trained. \cite{allen_-sample_2025} demonstrate that prediction methods with in-sample calibration guarantees can be used to construct sets of predictions with out-of-sample calibration guarantees, akin to conformal prediction. Connections between quantile recalibration and conformal prediction have been pointed out by \cite{DheurBenTaieb2023}, and since outcome-conditional recalibration yields forecasts that are in-sample calibrated, there is scope to \emph{conformalise} the approach to yield prediction sets that satisfy outcome-conditional calibration properties out-of-sample.

To our knowledge, outcome-conditional recalibration is the first post-hoc recalibration method that targets outcome-dependent notions of calibration. While auto-calibration (or distribution calibration) implies outcome-conditional calibration in theory, our empirical results suggest that recalibration schemes designed to enforce auto-calibration \citep[e.g.][]{kuleshov_calibrated_2022} do not necessarily yield predictive distributions that are calibrated when forecasting particular outcomes. In contrast, outcome-conditional recalibration does. Further work is therefore needed to clarify the connections between outcome-conditional calibration and auto-calibration.

\subsubsection*{Acknowledgements}
We thank Rafael Weinert for providing code and data that was used for the application to electricity price forecasts. 

\subsubsection*{Code and Data availability}
The code to run all experiments and reproduce results in this paper will be provided upon acceptance. 

\subsection*{AI Use Statement}
The authors used the generative AI models Claude Opus 4.8 and Claude Opus 5 in Claude Code to support with method implementation, as well as for editing of the manuscript. All outputs have been carefully checked and code has been tested for correctness. We take responsibility for the final content of this work, including text, claims, or artifacts produced with the aid of generative AI.

\bibliographystyle{apalike}
\bibliography{references}

\clearpage
\appendix
\thispagestyle{empty}

% Supplementary material: To improve readability, you must use a single-column format for the supplementary material.
\onecolumn
\aistatstitle{Supplementary Materials}
\
\section{Proofs}\label{app:proofs}
%We provide proofs of the propositions in the main paper. 

\begin{proof}[Proof of Proposition \ref{prop:proper}]
	The proof follows easily from the fact that $T_{\Ii_g}$ and $F^{\Ii_g}$ are continuous non-decreasing functions from zero to one, for all $g = 1, \dots, G$. That is, $T_{\Ii_g}$ is non-decreasing with $T_{\Ii_g}(0) = 0$ and $T_{\Ii_g}(1) = 1$, and $F^{\Ii_g}$ is non-decreasing with $F^{\Ii_g}(-\infty) = 0$ and $F^{\Ii_g}(+\infty) = 1$. Hence, 
	\[
	\Ftilde(-\infty) = m_0 + m_1 \left[ T_{\Ii_1} \circ F^{\Ii_g}(-\infty) \right] = m_1 T_{\Ii_1}(0) = 0
	\] 
	and 
	\[
	\Ftilde(+\infty) = \sum_{i=0}^{G-1} m_{i} + m_G \left[ T_{\Ii_G} \circ F^{\Ii_G}(+\infty) \right] = \sum_{i=0}^{G-1} m_{i} + m_G T_{\Ii_G}(1) = \sum_{i=0}^{G} m_{i} = 1.
	\]
	
	Since $T_{\Ii_g}$ and $F^{\Ii_g}$ are non-decreasing, so is $T_{\Ii_g} \circ F^{\Ii_g}$. Hence, for $x_1 \le x_2$ with $x_1, x_2 \in \Ii_g$, 
	\[
	\Ftilde(x_1) = \sum_{i=1}^{g-1} m_i + m_g \left[ T_{\Ii_g} \circ F^{\Ii_g}(x_1) \right] \le \sum_{i=1}^{g-1} m_i + m_g \left[ T_{\Ii_g} \circ F^{\Ii_g}(x_2) \right] = \Ftilde(x_2).
	\]
	If $x_1 \in \Ii_{g_1}$ and $x_2 \in \Ii_{g_2}$, for $g_1 < g_2$, then
	\[
	\Ftilde(x_1) = \sum_{i=1}^{g_1-1} m_i + m_{g_1} \left[ T_{\Ii_{g_1}} \circ F^{\Ii_{g_1}}(x_1) \right] \le \sum_{i=1}^{g_1} m_i \le \sum_{i=1}^{g_2 - 1} m_i + m_{g_2} \left[ T_{\Ii_{g_2}} \circ F^{\Ii_{g_2}}(x_2) \right] = \Ftilde(x_2).
	\]
	The recalibrated predictive distribution $\Ftilde$ is therefore non-decreasing, and bounded between zero and one. 
    
    Continuity within intervals $\Ii_g$ follows from the continuity of $T_{\Ii_g}$ and $F^{\Ii_g}$.  Let $\Ii_g = [a_{g-1}, a_g)$, at the limits one has:
    \[
    \underset{x \uparrow a_g}{\lim }\tilde F(x) = \sum_{i=1}^{g-1} m_i + m_g  = \sum_{i=1}^{g} m_i = \underset{x \downarrow a_g}{\lim }\tilde F(x),
    \]
    as $F^{\Ii_g}(a_{g-1}) = 0$, $F^{\Ii_g}(a_{g}) = 1$ and $T_{\Ii_g}(0) = 0$, $T_{\Ii_g}(1) = 1$.
\end{proof}

\vspace{0.1cm}
\begin{proof}[Proof of Proposition \ref{prop:tail_cpit}]
	For any $g = 1, \dots, G$, we have $F^{\Ii_g}(a_{g-1}) = 0$ and $F^{\Ii_g}(a_g) = 1$. Hence, for any $x \in \Ii_g$, we have
	\begin{align*}
		\tilde{F}^{\Ii_g}(x) &= \frac{\Ftilde(x) - \Ftilde(a_{g-1})}{\Ftilde(a_g) - \Ftilde(a_{g-1})}
		&= \frac{\sum_{i=0}^{g - 1} m_i + m_g \left[ T_{\Ii_g} \circ F^{\Ii_g}(x) \right] - \sum_{i=0}^{g - 1} m_i }{\sum_{i=0}^{g} m_i - \sum_{i=0}^{g - 1} m_i} 
		&= \frac{m_g \left[ T_{\Ii_g} \circ F^{\Ii_g}(x) \right] }{m_g} 
		= T_{\Ii_g} \circ F^{\Ii_g} (x).
	\end{align*}
	Conditional on $Y \in \Ii_g$, $\tilde{F}^{\Ii_g}(Y) = T_{\Ii_g} \circ F^{\Ii_g} (Y)$ follows a standard uniform distribution by construction. The recalibrated forecast distribution $\Ftilde$ therefore satisfies the first requirement of local calibration in Definition \ref{def:tail_cal_pop}, for all intervals $\Ii_1, \dots, \Ii_G$. The second requirement in Definition \ref{def:tail_cal_pop} follows since 
	\[
	\E[\Ftilde(a_g) - \Ftilde(a_{g-1})] = \E[m_g] = \E \left[ \frac{c_j [F(a_j) - F(a_{j-1})]}{ \sum_{g=1}^G c_g [F(a_g) - F(a_{g-1})]}  \right] = \Q(Y \in \Ii_j). \qedhere
	\]
\end{proof}

\section{Sample-level estimation}
\label{app:sample}

We provide a sample analysis and description of the method and implementation. 

\textbf{Setup.} Consider a partition of the outcome space $\R$ into $G$ contiguous intervals
\[
\R \;=\; \bigsqcup_{g=1}^{G} \mathcal{I}^{(g)},
\qquad
\mathcal{I}^{(g)} = [a^{(g)}, b^{(g)}),
\]
with $a^{(1)} = -\infty$ and $b^{(G)} = +\infty$. Write $N^{(g)} := |\{i : y_i \in \mathcal{I}^{(g)}\}|$ for the group sample size, $q^{(g)} := \Prob(Y \in \mathcal{I}^{(g)})$ for the occurrence probability (estimated by $N^{(g)}/N$), and $\Delta_i^{(g)} := F_i(b^{(g)}) - F_i(a^{(g)})$ for the forecast probability mass assigned to $\mathcal{I}^{(g)}$ by forecast $i$. Note that $\sum_{g=1}^{G}\Delta_i^{(g)} = 1$ for every $i$. Finally, write $F_i^{(g)}(\cdot) := [F_i(\cdot) - F_i(a^{(g)}]/\Delta_i^{(g)}$ for the CDF truncated on group $g$. 

\textbf{Intensity correction.} For each group $g = 1, \ldots, G$, compute the conditional PIT (CPIT) values
$u_i^{(g)} := F_i^{(g)}(y_i)$ for all $i$ with $y_i \in \mathcal{I}^{(g)}$,
and fit a non-decreasing map $T^{(g)} : [0,1] \to [0,1]$ via isotonic
regression on the tuples
\[
u_i^{(g)}
\;\;\longrightarrow\;\;
\frac{1}{N^{(g)}}\bigl|\bigl\{j : u_j^{(g)} \leq u_i^{(g)}\bigr\}\bigr|,
\]
with boundary conditions $T^{(g)}(0) = 0$ and $T^{(g)}(1) = 1$. These maps
correct probabilistic or quantile calibration on the group $g$.

\textbf{Occurrence correction.} To correct the probability mass assigned to each group, we introduce positive group scalings $c^{(1)}, \ldots, c^{(G)}$ and define the \emph{recalibrated group mass}
\begin{equation}
\label{eq:m}
m_i^{(g)}
\;=\;
\frac{c^{(g)}\,\Delta_i^{(g)}}
{\sum_{j=1}^{G} c^{(j)}\,\Delta_i^{(j)}}.
\end{equation}
By construction $\sum_{g=1}^{G} m_i^{(g)} = 1$ for every $i$, so the recalibrated forecast defined below is a proper CDF. The scalings are chosen so that the average recalibrated mass in each group matches the empirical occurrence rate,
\begin{equation}
\label{eq:occ_target}
\frac{1}{N}\sum_{i=1}^{N} m_i^{(g)} \;=\; q^{(g)},
\qquad g = 1, \ldots, G.
\end{equation}
Because the normalisation in eq.~\eqref{eq:m} couples all groups, eq.~\eqref{eq:occ_target} has no closed-form solution; instead $\{c^{(g)}\}$ is obtained by iterative proportional fitting (Sinkhorn's algorithm). Starting from $c^{(g)} = 1$, iterate
\begin{equation}
\label{eq:ipf}
c^{(g)} \;\leftarrow\; c^{(g)}\,\frac{q^{(g)}}{\frac{1}{N}\sum_{i=1}^{N} m_i^{(g)}},
\qquad g = 1, \ldots, G,
\end{equation}
recomputing $m_i^{(g)}$ from~\eqref{eq:m} after each sweep, until convergence. The overall scale of $\{c^{(g)}\}$ cancels in~\eqref{eq:m} and may be fixed by normalising $\sum_g c^{(g)} = 1$. Setting $c^{(g)} \equiv 1$ recovers the uncorrected case $m_i^{(g)} = \Delta_i^{(g)}$.

The recalibrated CDF is then obtained as combination of the occurrence and intensity correction:

\textbf{Recalibrated CDF.} For a new forecast--observation pair $(F_i, y_i)$, using the scalings $\{c^{(g)}\}$ learned on the training set, define the group-specific intercept
\begin{equation}
\label{eq:A}
A_i^{(g)} \;=\; \sum_{j=1}^{g-1} m_i^{(j)}.
\end{equation}
The \emph{outcome-conditional recalibrated CDF} for $y \in \mathcal{I}^{(g)}$ is
\begin{equation}
\label{eq:tail_recal}
\Ftilde_i(y)
\;=\;
A_i^{(g)} \;+\; m_i^{(g)}\,T^{(g)}\!\left(u_i^{(g)}\right),
\end{equation}
where $u_i^{(g)} = \bigl(F_i(y) - F_i(a^{(g)})\bigr) / \Delta_i^{(g)}$ is the CPIT value of the new observation. Here $A_i^{(g)}$ is the recalibrated value of $F_i(a^{(g)})$ and $m_i^{(g)}$ is the recalibrated probability mass assigned to $\mathcal{I}^{(g)}$. This corresponds to~\eqref{eq:recal_pop}. Validity, continuity and in-sample outcome-conditional calibration of the resulting CDF is established by Propositions \ref{prop:proper} and \ref{prop:tail_cpit}.

\section{Per-dataset UCI/OpenML results}\label{sec:per_dataset_uci}

The following ten UCI and OpenML datasets were used in Section \ref{sec:res_bench}: energy ($N = 768$, UCI-ML id = 242), kinematics ($N = 8192$, Delve kin8nm), power ($N = 9568$, UCI-ML id = 294), airfoil ($N = 1503$, UCI-ML id = 291), superconductivity ($N = 21263$, UCI-ML id = 464), protein ($N = 45730$, UCI-ML id = 265), ames ($N = 2930$, OpenML id = 43926), bank ($N = 8192$, OpenML bank8FM), elevators ($N = 16599$, OpenML elevators) and puma8NH ($N = 8192$, OpenML puma8NH). All experiments were run on a modern laptop (Apple M4 chip, 16Gb RAM).

In this Appendix, we present results per dataset and per model. Table \ref{tab:subset_in_sample} shows outcome-conditional (CPIT/subset) calibration error and occurrence error for each of the competing recalibration methods (uncalibrated, KFE18, KD22, outcome-conditional recalibration: OC, and outcome-conditional + quantile recalibration: OC+QR) and for each dataset and base method (Bayesian Ridge and MC Dropout). Table \ref{tab:overall_in_sample} shows overall calibration error. We find that outcome-conditional recalibration outperforms all other methods in terms of CPIT/subset calibration error and occurrence error, in all ten datasets. In terms of overall calibration, KFE18 and outcome-conditional + quantile recalibration appear best for most datasets. The OC+QR method seems to be a good intermediate that is close on subset and occurrence calibration to the outcome-conditional method, but also has good overall calibration.

% ===================== calibration-error variant: binned =====================
% Requires: \usepackage{booktabs}
\begin{table*}[h]
\centering
\scriptsize
\setlength{\tabcolsep}{2.5pt}
\caption{Outcome-conditional (CPIT/subset) calibration error and occurrence error across datasets and methods. Values report the mean $\pm$ sample standard deviation over ten seeds, scaled by $10^{2}$. Lower is better. Bold indicates the lowest mean among methods for a fixed base model, dataset, and metric.\\}
\label{tab:subset_in_sample}
{\normalsize Bayesian ridge}\par
\begin{tabular}{lrrrrrrrrrr}
\toprule
Dataset & \multicolumn{2}{c}{Uncal.} & \multicolumn{2}{c}{Kuleshov et al. '18} & \multicolumn{2}{c}{Kuleshov--Deshpande '22} & \multicolumn{2}{c}{Ours: OC} & \multicolumn{2}{c}{Ours: OC + QR} \\
\cmidrule(lr){2-3} \cmidrule(lr){4-5} \cmidrule(lr){6-7} \cmidrule(lr){8-9} \cmidrule(lr){10-11}
 & Cal-sub. $\downarrow$ & Occ-sub. $\downarrow$ & Cal-sub. $\downarrow$ & Occ-sub. $\downarrow$ & Cal-sub. $\downarrow$ & Occ-sub. $\downarrow$ & Cal-sub. $\downarrow$ & Occ-sub. $\downarrow$ & Cal-sub. $\downarrow$ & Occ-sub. $\downarrow$ \\
\midrule
airfoil & $2.60 \pm 0.47$ & $1.93 \pm 0.34$ & $3.10 \pm 0.58$ & $2.18 \pm 0.37$ & $1.89 \pm 0.69$ & $1.19 \pm 0.28$ & $\mathbf{1.39 \pm 0.35}$ & $\mathbf{0.54 \pm 0.30}$ & $1.66 \pm 0.29$ & $0.65 \pm 0.32$ \\
ames & $7.28 \pm 0.96$ & $2.56 \pm 0.24$ & $6.16 \pm 0.90$ & $2.19 \pm 0.28$ & $1.40 \pm 0.60$ & $0.62 \pm 0.21$ & $\mathbf{1.06 \pm 0.41}$ & $\mathbf{0.43 \pm 0.16}$ & $2.79 \pm 0.83$ & $1.02 \pm 0.24$ \\
bank & $32.26 \pm 1.32$ & $3.74 \pm 0.16$ & $30.32 \pm 1.58$ & $3.73 \pm 0.16$ & $1.67 \pm 1.03$ & $0.44 \pm 0.14$ & $\mathbf{0.43 \pm 0.14}$ & $\mathbf{0.27 \pm 0.09}$ & $1.31 \pm 0.46$ & $0.34 \pm 0.12$ \\
elevators & $26.56 \pm 0.56$ & $4.79 \pm 0.15$ & $19.05 \pm 0.60$ & $3.03 \pm 0.16$ & $13.82 \pm 0.92$ & $0.54 \pm 0.21$ & $\mathbf{0.20 \pm 0.06}$ & $\mathbf{0.28 \pm 0.09}$ & $0.42 \pm 0.12$ & $0.66 \pm 0.13$ \\
energy & $29.09 \pm 5.59$ & $3.79 \pm 0.46$ & $27.02 \pm 6.43$ & $3.78 \pm 0.47$ & $9.16 \pm 3.68$ & $1.22 \pm 0.49$ & $\mathbf{3.46 \pm 1.03}$ & $\mathbf{0.91 \pm 0.20}$ & $5.26 \pm 1.90$ & $1.12 \pm 0.40$ \\
kinematics & $0.34 \pm 0.07$ & $1.06 \pm 0.10$ & $1.21 \pm 0.34$ & $1.37 \pm 0.10$ & $0.34 \pm 0.06$ & $0.34 \pm 0.17$ & $\mathbf{0.25 \pm 0.08}$ & $\mathbf{0.22 \pm 0.11}$ & $0.93 \pm 0.28$ & $0.70 \pm 0.21$ \\
power & $3.26 \pm 0.49$ & $1.21 \pm 0.11$ & $2.94 \pm 0.47$ & $1.19 \pm 0.12$ & $0.45 \pm 0.12$ & $0.33 \pm 0.11$ & $\mathbf{0.32 \pm 0.07}$ & $\mathbf{0.28 \pm 0.12}$ & $0.38 \pm 0.08$ & $0.28 \pm 0.13$ \\
protein & $44.51 \pm 0.44$ & $8.98 \pm 0.05$ & $32.81 \pm 0.53$ & $7.57 \pm 0.06$ & $0.52 \pm 0.28$ & $0.41 \pm 0.19$ & $\mathbf{0.06 \pm 0.02}$ & $\mathbf{0.07 \pm 0.02}$ & $0.10 \pm 0.03$ & $0.30 \pm 0.04$ \\
puma8NH & $7.75 \pm 0.21$ & $3.77 \pm 0.09$ & $3.85 \pm 0.31$ & $3.86 \pm 0.09$ & $0.83 \pm 0.29$ & $0.85 \pm 0.23$ & $\mathbf{0.24 \pm 0.06}$ & $\mathbf{0.19 \pm 0.12}$ & $0.27 \pm 0.07$ & $0.21 \pm 0.13$ \\
supercon & $49.72 \pm 0.74$ & $6.87 \pm 0.10$ & $50.22 \pm 0.78$ & $6.85 \pm 0.10$ & $1.14 \pm 0.88$ & $0.72 \pm 0.21$ & $\mathbf{0.13 \pm 0.04}$ & $\mathbf{0.20 \pm 0.08}$ & $0.35 \pm 0.05$ & $0.82 \pm 0.11$ \\
\bottomrule
\end{tabular}

\vspace{0.7em}

{\normalsize MC dropout}\par
\begin{tabular}{lrrrrrrrrrr}
\toprule
Dataset & \multicolumn{2}{c}{Uncal.} & \multicolumn{2}{c}{Kuleshov et al.'18} & \multicolumn{2}{c}{Kuleshov--Deshpande '22} & \multicolumn{2}{c}{Ours: OC} & \multicolumn{2}{c}{Ours: OC + QR} \\
\cmidrule(lr){2-3} \cmidrule(lr){4-5} \cmidrule(lr){6-7} \cmidrule(lr){8-9} \cmidrule(lr){10-11}
 & Cal-sub. $\downarrow$ & Occ-sub. $\downarrow$ & Cal-sub. $\downarrow$ & Occ-sub. $\downarrow$ & Cal-sub. $\downarrow$ & Occ-sub. $\downarrow$ & Cal-sub. $\downarrow$ & Occ-sub. $\downarrow$ & Cal-sub. $\downarrow$ & Occ-sub. $\downarrow$ \\
\midrule
airfoil & $3.33 \pm 1.21$ & $2.46 \pm 0.67$ & $5.40 \pm 1.45$ & $3.21 \pm 0.65$ & $2.19 \pm 0.37$ & $1.07 \pm 0.32$ & $\mathbf{1.53 \pm 0.39}$ & $\mathbf{0.59 \pm 0.30}$ & $3.60 \pm 1.64$ & $1.01 \pm 0.50$ \\
ames & $8.69 \pm 8.97$ & $2.84 \pm 1.65$ & $9.84 \pm 2.63$ & $1.20 \pm 0.55$ & $6.74 \pm 1.78$ & $0.74 \pm 0.27$ & $\mathbf{2.74 \pm 1.12}$ & $\mathbf{0.46 \pm 0.19}$ & $13.16 \pm 2.66$ & $2.58 \pm 1.48$ \\
bank & $19.25 \pm 1.74$ & $1.81 \pm 0.43$ & $14.42 \pm 1.66$ & $1.71 \pm 0.34$ & $1.85 \pm 1.55$ & $0.42 \pm 0.19$ & $\mathbf{0.55 \pm 0.24}$ & $\mathbf{0.29 \pm 0.08}$ & $0.81 \pm 0.28$ & $0.31 \pm 0.08$ \\
elevators & $16.14 \pm 1.68$ & $1.72 \pm 0.86$ & $14.76 \pm 2.01$ & $1.04 \pm 0.59$ & $12.43 \pm 1.20$ & $0.48 \pm 0.16$ & $\mathbf{0.26 \pm 0.11}$ & $\mathbf{0.32 \pm 0.12}$ & $0.39 \pm 0.18$ & $0.47 \pm 0.19$ \\
energy & $19.20 \pm 3.55$ & $3.66 \pm 1.33$ & $16.10 \pm 4.21$ & $3.24 \pm 1.42$ & $8.86 \pm 2.45$ & $1.30 \pm 0.44$ & $\mathbf{3.70 \pm 1.15}$ & $\mathbf{0.81 \pm 0.17}$ & $7.06 \pm 1.92$ & $1.14 \pm 0.51$ \\
kinematics & $4.11 \pm 1.16$ & $1.35 \pm 0.29$ & $4.86 \pm 1.03$ & $1.41 \pm 0.27$ & $0.50 \pm 0.15$ & $0.37 \pm 0.16$ & $\mathbf{0.41 \pm 0.11}$ & $\mathbf{0.30 \pm 0.11}$ & $1.18 \pm 0.28$ & $0.36 \pm 0.14$ \\
power & $1.32 \pm 0.63$ & $0.86 \pm 0.25$ & $1.32 \pm 0.66$ & $0.92 \pm 0.28$ & $0.49 \pm 0.15$ & $0.33 \pm 0.12$ & $\mathbf{0.34 \pm 0.10}$ & $0.27 \pm 0.09$ & $0.45 \pm 0.08$ & $\mathbf{0.26 \pm 0.09}$ \\
protein & $43.38 \pm 0.62$ & $9.20 \pm 0.40$ & $37.33 \pm 1.33$ & $8.97 \pm 0.25$ & $0.44 \pm 0.17$ & $0.61 \pm 0.18$ & $\mathbf{0.07 \pm 0.02}$ & $\mathbf{0.12 \pm 0.06}$ & $0.23 \pm 0.07$ & $0.52 \pm 0.09$ \\
puma8NH & $5.97 \pm 0.66$ & $3.49 \pm 0.39$ & $4.50 \pm 0.26$ & $3.66 \pm 0.46$ & $0.63 \pm 0.37$ & $0.54 \pm 0.14$ & $\mathbf{0.35 \pm 0.11}$ & $\mathbf{0.27 \pm 0.15}$ & $0.58 \pm 0.29$ & $0.44 \pm 0.16$ \\
supercon & $38.35 \pm 1.96$ & $5.70 \pm 0.50$ & $27.14 \pm 2.21$ & $5.46 \pm 0.59$ & $0.51 \pm 0.15$ & $0.52 \pm 0.28$ & $\mathbf{0.14 \pm 0.04}$ & $\mathbf{0.19 \pm 0.06}$ & $0.73 \pm 0.19$ & $1.05 \pm 0.14$ \\
\bottomrule
\end{tabular}
\end{table*}

\begin{table*}[h]
\centering
\scriptsize
\setlength{\tabcolsep}{6pt}
\caption{Overall calibration error across datasets and methods. Values report the mean $\pm$ sample standard deviation over ten seeds, scaled by $10^{2}$. Lower is better. Bold indicates the lowest mean among methods for a fixed base model, dataset, and metric.\\}
\label{tab:overall_in_sample}
{\normalsize Bayesian ridge}\par
\begin{tabular}{lrrrrr}
\toprule
Dataset & Uncal. & Kuleshov et al. '18 & Kuleshov--Deshpande '22 & Ours: OC & Ours: OC + QR \\
\midrule
airfoil & $0.37 \pm 0.21$ & $\mathbf{0.31 \pm 0.13}$ & $0.55 \pm 0.45$ & $0.70 \pm 0.38$ & $0.36 \pm 0.28$ \\
ames & $3.19 \pm 0.74$ & $0.58 \pm 0.28$ & $0.55 \pm 0.41$ & $1.97 \pm 0.45$ & $\mathbf{0.42 \pm 0.22}$ \\
bank & $1.16 \pm 0.46$ & $0.15 \pm 0.19$ & $0.41 \pm 0.29$ & $1.18 \pm 0.33$ & $\mathbf{0.15 \pm 0.20}$ \\
elevators & $2.69 \pm 0.38$ & $\mathbf{0.07 \pm 0.06}$ & $0.18 \pm 0.18$ & $0.26 \pm 0.10$ & $0.07 \pm 0.05$ \\
energy & $3.39 \pm 1.58$ & $0.96 \pm 0.59$ & $2.08 \pm 1.21$ & $1.98 \pm 1.00$ & $\mathbf{0.50 \pm 0.37}$ \\
kinematics & $0.75 \pm 0.26$ & $\mathbf{0.08 \pm 0.05}$ & $0.18 \pm 0.09$ & $0.90 \pm 0.34$ & $0.08 \pm 0.05$ \\
power & $0.19 \pm 0.07$ & $\mathbf{0.09 \pm 0.06}$ & $0.13 \pm 0.09$ & $0.16 \pm 0.10$ & $0.09 \pm 0.06$ \\
protein & $1.76 \pm 0.15$ & $0.01 \pm 0.01$ & $0.03 \pm 0.02$ & $0.05 \pm 0.02$ & $\mathbf{0.01 \pm 0.01}$ \\
puma8NH & $0.26 \pm 0.10$ & $0.07 \pm 0.06$ & $0.20 \pm 0.18$ & $0.13 \pm 0.06$ & $\mathbf{0.06 \pm 0.05}$ \\
supercon & $0.56 \pm 0.15$ & $0.04 \pm 0.03$ & $0.22 \pm 0.17$ & $0.47 \pm 0.16$ & $\mathbf{0.03 \pm 0.02}$ \\
\bottomrule
\end{tabular}

\vspace{0.7em}

{\normalsize MC dropout}\par
\begin{tabular}{lrrrrr}
\toprule
Dataset & Uncal. & Kuleshov et al. '18 & Kuleshov--Deshpande '22 & Ours: OC & Ours: OC + QR \\
\midrule
airfoil & $1.97 \pm 1.05$ & $\mathbf{0.46 \pm 0.50}$ & $0.90 \pm 0.54$ & $1.69 \pm 0.46$ & $0.63 \pm 0.64$ \\
ames & $12.96 \pm 12.25$ & $8.00 \pm 1.31$ & $9.44 \pm 0.66$ & $\mathbf{3.20 \pm 2.66}$ & $4.47 \pm 2.46$ \\
bank & $1.20 \pm 0.54$ & $0.20 \pm 0.21$ & $0.58 \pm 0.64$ & $0.23 \pm 0.21$ & $\mathbf{0.20 \pm 0.20}$ \\
elevators & $1.46 \pm 0.97$ & $\mathbf{0.06 \pm 0.03}$ & $0.07 \pm 0.05$ & $0.17 \pm 0.17$ & $0.07 \pm 0.04$ \\
energy & $5.34 \pm 3.17$ & $0.76 \pm 0.48$ & $1.63 \pm 1.68$ & $2.53 \pm 1.20$ & $\mathbf{0.63 \pm 0.32}$ \\
kinematics & $0.34 \pm 0.47$ & $0.17 \pm 0.08$ & $\mathbf{0.16 \pm 0.11}$ & $0.28 \pm 0.17$ & $0.20 \pm 0.05$ \\
power & $0.43 \pm 0.45$ & $0.10 \pm 0.04$ & $0.13 \pm 0.07$ & $0.10 \pm 0.07$ & $\mathbf{0.09 \pm 0.04}$ \\
protein & $1.51 \pm 0.45$ & $\mathbf{0.02 \pm 0.02}$ & $0.11 \pm 0.10$ & $0.34 \pm 0.17$ & $0.03 \pm 0.02$ \\
puma8NH & $0.37 \pm 0.22$ & $\mathbf{0.14 \pm 0.06}$ & $0.20 \pm 0.16$ & $0.17 \pm 0.06$ & $0.14 \pm 0.05$ \\
supercon & $3.41 \pm 0.86$ & $0.04 \pm 0.02$ & $0.12 \pm 0.05$ & $0.76 \pm 0.15$ & $\mathbf{0.03 \pm 0.03}$ \\
\bottomrule
\end{tabular}
\end{table*}

\clearpage
\section{Out-of-sample fitting for the UCI/OpenML experiments}\label{sec:out_of_sample_uci}

We also tested fitting the recalibration methods out-of-sample, as \citet{kuleshov_calibrated_2022} note that their conditional recalibration method might not perform well when trained in-sample. As before, we reserve 75\% of the data for training and 25\% for testing. Of the training dataset, we use 30\% (up to 2000 data points) to fit the recalibration models, whilst using the rest to estimate the base models. We then evaluate on the remaining 25\%. Table \ref{tab:subset_out_of_sample} contains the results in terms of subset calibration error and occurrence error, Table \ref{tab:overall_out_of_sample} presents overall calibration errors. The results are similar to in-sample fitting. The outcome-conditional recalibration method performs best on almost all datasets in terms of outcome-conditional calibration; the exceptions are the \texttt{ames} dataset for the MC-Dropout model, where KD22 marginally outperforms outcome-conditional recalibration, and
the \texttt{kinematics} dataset for the Bayesian ridge model, for which the base model is already very well calibrated. In terms of overall calibration error, outcome-conditional+quantile recalibration and the quantile recalibration of \citet{kuleshov_accurate_2018} again issue the most calibrated predictive distributions for all datasets. 

% ===================== calibration-error variant: binned =====================
% Requires: \usepackage{booktabs}
\begin{table*}[h]
\centering
\scriptsize
\setlength{\tabcolsep}{2.5pt}
\caption{Outcome-conditional calibration (CPIT/subset) error and occurrence error across datasets and methods. Values report the mean $\pm$ sample standard deviation over ten seeds, scaled by $10^{2}$. Lower is better. Bold indicates the lowest mean among methods for a fixed base model, dataset, and metric.\\}
\label{tab:subset_out_of_sample}
{\normalsize Bayesian ridge}\par
\begin{tabular}{lrrrrrrrrrr}
\toprule
Dataset & \multicolumn{2}{c}{Uncal.} & \multicolumn{2}{c}{Kuleshov et al. '18} & \multicolumn{2}{c}{Kuleshov--Deshpande '22} & \multicolumn{2}{c}{Ours: OC} & \multicolumn{2}{c}{Ours: OC + QR} \\
\cmidrule(lr){2-3} \cmidrule(lr){4-5} \cmidrule(lr){6-7} \cmidrule(lr){8-9} \cmidrule(lr){10-11}
 & Cal-sub. $\downarrow$ & Occ-sub. $\downarrow$ & Cal-sub. $\downarrow$ & Occ-sub. $\downarrow$ & Cal-sub. $\downarrow$ & Occ-sub. $\downarrow$ & Cal-sub. $\downarrow$ & Occ-sub. $\downarrow$ & Cal-sub. $\downarrow$ & Occ-sub. $\downarrow$ \\
\midrule
airfoil & $2.62 \pm 0.56$ & $1.95 \pm 0.40$ & $3.20 \pm 0.98$ & $2.24 \pm 0.49$ & $3.16 \pm 1.30$ & $1.49 \pm 0.35$ & $\mathbf{2.10 \pm 0.31}$ & $\mathbf{0.75 \pm 0.37}$ & $2.29 \pm 0.51$ & $0.85 \pm 0.38$ \\
ames & $7.69 \pm 1.09$ & $2.59 \pm 0.20$ & $6.68 \pm 1.26$ & $2.37 \pm 0.24$ & $2.05 \pm 0.79$ & $0.76 \pm 0.20$ & $\mathbf{1.81 \pm 0.72}$ & $\mathbf{0.52 \pm 0.20}$ & $2.66 \pm 1.01$ & $0.84 \pm 0.31$ \\
bank & $32.19 \pm 1.34$ & $3.75 \pm 0.18$ & $30.45 \pm 1.72$ & $3.73 \pm 0.17$ & $2.93 \pm 1.84$ & $0.55 \pm 0.18$ & $\mathbf{0.69 \pm 0.33}$ & $\mathbf{0.30 \pm 0.11}$ & $1.64 \pm 0.80$ & $0.37 \pm 0.13$ \\
elevators & $26.56 \pm 0.65$ & $4.79 \pm 0.14$ & $18.77 \pm 0.74$ & $3.05 \pm 0.23$ & $14.37 \pm 0.95$ & $0.66 \pm 0.26$ & $\mathbf{0.49 \pm 0.22}$ & $\mathbf{0.41 \pm 0.16}$ & $0.83 \pm 0.38$ & $0.80 \pm 0.22$ \\
energy & $29.13 \pm 5.55$ & $3.83 \pm 0.43$ & $26.12 \pm 7.46$ & $3.82 \pm 0.47$ & $11.42 \pm 5.93$ & $1.65 \pm 0.73$ & $\mathbf{6.36 \pm 2.49}$ & $\mathbf{1.37 \pm 0.58}$ & $9.04 \pm 3.54$ & $1.81 \pm 0.77$ \\
kinematics & $\mathbf{0.33 \pm 0.07}$ & $1.06 \pm 0.14$ & $1.28 \pm 0.34$ & $1.34 \pm 0.12$ & $0.49 \pm 0.17$ & $0.51 \pm 0.22$ & $0.35 \pm 0.10$ & $\mathbf{0.27 \pm 0.14}$ & $1.04 \pm 0.35$ & $0.62 \pm 0.31$ \\
power & $3.27 \pm 0.49$ & $1.22 \pm 0.13$ & $2.99 \pm 0.56$ & $1.22 \pm 0.14$ & $0.62 \pm 0.21$ & $0.43 \pm 0.14$ & $\mathbf{0.53 \pm 0.19}$ & $\mathbf{0.30 \pm 0.14}$ & $0.58 \pm 0.16$ & $0.32 \pm 0.13$ \\
protein & $44.52 \pm 0.45$ & $8.98 \pm 0.05$ & $32.42 \pm 1.19$ & $7.78 \pm 0.20$ & $1.60 \pm 0.70$ & $0.84 \pm 0.30$ & $\mathbf{0.23 \pm 0.05}$ & $\mathbf{0.20 \pm 0.14}$ & $0.30 \pm 0.12$ & $0.47 \pm 0.14$ \\
puma8NH & $7.72 \pm 0.16$ & $3.77 \pm 0.08$ & $3.93 \pm 0.60$ & $3.83 \pm 0.15$ & $0.82 \pm 0.28$ & $0.96 \pm 0.26$ & $\mathbf{0.42 \pm 0.11}$ & $\mathbf{0.27 \pm 0.18}$ & $0.42 \pm 0.17$ & $0.28 \pm 0.15$ \\
supercon & $49.79 \pm 0.71$ & $6.88 \pm 0.10$ & $49.95 \pm 0.72$ & $6.91 \pm 0.10$ & $1.01 \pm 0.46$ & $0.78 \pm 0.39$ & $\mathbf{0.39 \pm 0.11}$ & $\mathbf{0.28 \pm 0.11}$ & $0.57 \pm 0.10$ & $0.83 \pm 0.12$ \\
\bottomrule
\end{tabular}

\vspace{0.7em}

{\normalsize MC dropout}\par
\begin{tabular}{lrrrrrrrrrr}
\toprule
Dataset & \multicolumn{2}{c}{Uncal.} & \multicolumn{2}{c}{Kuleshov et al. '18} & \multicolumn{2}{c}{Kuleshov--Deshpande '22} & \multicolumn{2}{c}{Ours: OC} & \multicolumn{2}{c}{Ours: OC + QR} \\
\cmidrule(lr){2-3} \cmidrule(lr){4-5} \cmidrule(lr){6-7} \cmidrule(lr){8-9} \cmidrule(lr){10-11}
 & Cal-sub. $\downarrow$ & Occ-sub. $\downarrow$ & Cal-sub. $\downarrow$ & Occ-sub. $\downarrow$ & Cal-sub. $\downarrow$ & Occ-sub. $\downarrow$ & Cal-sub. $\downarrow$ & Occ-sub. $\downarrow$ & Cal-sub. $\downarrow$ & Occ-sub. $\downarrow$ \\
\midrule
airfoil & $4.70 \pm 2.26$ & $3.32 \pm 0.82$ & $4.71 \pm 1.73$ & $3.40 \pm 0.68$ & $3.15 \pm 1.43$ & $1.51 \pm 0.56$ & $\mathbf{2.54 \pm 0.82}$ & $\mathbf{0.92 \pm 0.37}$ & $2.94 \pm 0.91$ & $1.11 \pm 0.41$ \\
ames & $5.08 \pm 5.81$ & $2.11 \pm 0.95$ & $2.71 \pm 1.30$ & $1.42 \pm 0.28$ & $\mathbf{1.55 \pm 0.74}$ & $0.68 \pm 0.23$ & $1.74 \pm 0.46$ & $\mathbf{0.47 \pm 0.20}$ & $4.22 \pm 2.21$ & $1.24 \pm 0.79$ \\
bank & $19.13 \pm 2.19$ & $2.33 \pm 0.47$ & $15.18 \pm 1.91$ & $2.02 \pm 0.42$ & $2.31 \pm 2.59$ & $0.58 \pm 0.27$ & $\mathbf{0.79 \pm 0.34}$ & $\mathbf{0.34 \pm 0.14}$ & $1.00 \pm 0.43$ & $0.35 \pm 0.12$ \\
elevators & $16.06 \pm 1.38$ & $1.56 \pm 0.70$ & $14.17 \pm 2.23$ & $0.93 \pm 0.46$ & $12.36 \pm 0.86$ & $0.56 \pm 0.23$ & $\mathbf{0.43 \pm 0.13}$ & $\mathbf{0.45 \pm 0.15}$ & $0.57 \pm 0.13$ & $0.56 \pm 0.20$ \\
energy & $17.56 \pm 5.23$ & $2.98 \pm 0.77$ & $15.28 \pm 5.25$ & $2.69 \pm 0.83$ & $10.73 \pm 3.87$ & $1.55 \pm 0.52$ & $\mathbf{6.45 \pm 1.96}$ & $\mathbf{1.33 \pm 0.53}$ & $8.67 \pm 2.24$ & $1.61 \pm 0.72$ \\
kinematics & $4.01 \pm 1.05$ & $1.37 \pm 0.36$ & $3.72 \pm 0.89$ & $1.37 \pm 0.37$ & $0.57 \pm 0.26$ & $0.52 \pm 0.22$ & $\mathbf{0.52 \pm 0.17}$ & $\mathbf{0.35 \pm 0.12}$ & $0.79 \pm 0.35$ & $0.38 \pm 0.14$ \\
power & $1.16 \pm 0.54$ & $0.90 \pm 0.46$ & $0.85 \pm 0.61$ & $0.88 \pm 0.45$ & $0.65 \pm 0.22$ & $0.38 \pm 0.11$ & $\mathbf{0.54 \pm 0.14}$ & $0.30 \pm 0.15$ & $0.65 \pm 0.22$ & $\mathbf{0.29 \pm 0.14}$ \\
protein & $43.20 \pm 0.60$ & $9.38 \pm 0.51$ & $36.63 \pm 2.07$ & $9.09 \pm 0.40$ & $1.06 \pm 0.81$ & $0.87 \pm 0.25$ & $\mathbf{0.28 \pm 0.08}$ & $\mathbf{0.25 \pm 0.14}$ & $0.46 \pm 0.25$ & $0.66 \pm 0.22$ \\
puma8NH & $6.10 \pm 0.74$ & $3.25 \pm 0.27$ & $4.77 \pm 0.48$ & $3.34 \pm 0.28$ & $0.59 \pm 0.35$ & $0.71 \pm 0.33$ & $\mathbf{0.44 \pm 0.17}$ & $\mathbf{0.36 \pm 0.18}$ & $0.53 \pm 0.22$ & $0.45 \pm 0.16$ \\
supercon & $40.22 \pm 2.64$ & $5.58 \pm 0.70$ & $29.06 \pm 2.56$ & $5.29 \pm 0.95$ & $0.98 \pm 0.46$ & $0.57 \pm 0.23$ & $\mathbf{0.43 \pm 0.15}$ & $\mathbf{0.25 \pm 0.06}$ & $0.87 \pm 0.26$ & $0.95 \pm 0.19$ \\
\bottomrule
\end{tabular}
\end{table*}

\begin{table*}[h]
\centering
\scriptsize
\setlength{\tabcolsep}{6pt}
\caption{Overall calibration error across datasets and methods. Values report the mean $\pm$ sample standard deviation over ten seeds, scaled by $10^{2}$. Lower is better. Bold indicates the lowest mean among methods for a fixed base model, dataset, and metric.\\}
\label{tab:overall_out_of_sample}
{\normalsize Bayesian ridge}\par
\begin{tabular}{lrrrrr}
\toprule
Dataset & Uncal. & Kuleshov et al. '18 & Kuleshov--Deshpande '22 & Ours: OC & Ours: OC + QR \\
\midrule
airfoil & $\mathbf{0.37 \pm 0.23}$ & $0.51 \pm 0.38$ & $1.30 \pm 0.70$ & $0.95 \pm 0.66$ & $0.60 \pm 0.62$ \\
ames & $2.74 \pm 1.02$ & $\mathbf{0.49 \pm 0.31}$ & $0.66 \pm 0.99$ & $1.80 \pm 0.47$ & $0.50 \pm 0.39$ \\
bank & $1.14 \pm 0.45$ & $\mathbf{0.16 \pm 0.19}$ & $0.58 \pm 0.36$ & $1.17 \pm 0.42$ & $0.21 \pm 0.28$ \\
elevators & $2.71 \pm 0.37$ & $0.11 \pm 0.10$ & $0.14 \pm 0.10$ & $0.29 \pm 0.11$ & $\mathbf{0.11 \pm 0.10}$ \\
energy & $3.39 \pm 1.66$ & $1.45 \pm 1.56$ & $2.84 \pm 1.40$ & $2.71 \pm 1.40$ & $\mathbf{1.33 \pm 0.99}$ \\
kinematics & $0.71 \pm 0.26$ & $\mathbf{0.14 \pm 0.10}$ & $0.20 \pm 0.14$ & $1.03 \pm 0.40$ & $0.16 \pm 0.09$ \\
power & $0.18 \pm 0.05$ & $\mathbf{0.16 \pm 0.11}$ & $0.19 \pm 0.08$ & $0.22 \pm 0.19$ & $0.17 \pm 0.12$ \\
protein & $1.75 \pm 0.15$ & $\mathbf{0.06 \pm 0.05}$ & $0.25 \pm 0.23$ & $0.07 \pm 0.03$ & $0.07 \pm 0.05$ \\
puma8NH & $0.27 \pm 0.11$ & $\mathbf{0.10 \pm 0.04}$ & $0.29 \pm 0.19$ & $0.15 \pm 0.11$ & $0.11 \pm 0.06$ \\
supercon & $0.57 \pm 0.16$ & $0.11 \pm 0.07$ & $0.21 \pm 0.13$ & $0.52 \pm 0.24$ & $\mathbf{0.09 \pm 0.06}$ \\
\bottomrule
\end{tabular}

\vspace{0.7em}

{\normalsize MC dropout}\par
\begin{tabular}{lrrrrr}
\toprule
Dataset & Uncal. & Kuleshov et al. '18 & Kuleshov--Deshpande '22 & Ours: OC & Ours: OC + QR \\
\midrule
airfoil & $1.04 \pm 0.59$ & $\mathbf{0.62 \pm 0.46}$ & $1.86 \pm 1.43$ & $1.72 \pm 0.66$ & $0.67 \pm 0.43$ \\
ames & $6.87 \pm 7.03$ & $0.41 \pm 0.29$ & $0.50 \pm 0.52$ & $3.41 \pm 2.17$ & $\mathbf{0.39 \pm 0.18}$ \\
bank & $2.05 \pm 1.66$ & $\mathbf{0.20 \pm 0.20}$ & $0.57 \pm 0.55$ & $0.34 \pm 0.29$ & $0.20 \pm 0.23$ \\
elevators & $0.94 \pm 0.55$ & $0.08 \pm 0.04$ & $0.15 \pm 0.12$ & $0.16 \pm 0.10$ & $\mathbf{0.08 \pm 0.04}$ \\
energy & $4.64 \pm 2.58$ & $\mathbf{1.29 \pm 0.78}$ & $2.32 \pm 1.27$ & $2.32 \pm 1.09$ & $1.45 \pm 0.80$ \\
kinematics & $0.24 \pm 0.13$ & $\mathbf{0.14 \pm 0.09}$ & $0.24 \pm 0.18$ & $0.50 \pm 0.23$ & $0.14 \pm 0.09$ \\
power & $0.72 \pm 0.41$ & $\mathbf{0.13 \pm 0.08}$ & $0.24 \pm 0.10$ & $0.21 \pm 0.15$ & $0.14 \pm 0.06$ \\
protein & $1.82 \pm 0.64$ & $\mathbf{0.06 \pm 0.05}$ & $0.38 \pm 0.35$ & $0.35 \pm 0.18$ & $0.08 \pm 0.07$ \\
puma8NH & $0.33 \pm 0.18$ & $\mathbf{0.16 \pm 0.18}$ & $0.30 \pm 0.32$ & $0.25 \pm 0.17$ & $0.16 \pm 0.17$ \\
supercon & $4.17 \pm 1.59$ & $0.08 \pm 0.06$ & $0.32 \pm 0.27$ & $0.83 \pm 0.17$ & $\mathbf{0.07 \pm 0.04}$ \\
\bottomrule
\end{tabular}
\end{table*}

\clearpage
\section{Evaluating on partitions not used for fitting} \label{sec:partition_shift}

In the results in the main paper the, evaluation of outcome-conditional calibration was always performed with respect to the partitions that were used for recalibration. In this Appendix, we evaluate to what extent outcome-conditional recalibration transfers across partitions. We assess how a method fitted on $G$ partitions transfers when evaluated on finer partitions, and also assess performance with different combinations of partitions.

Table \ref{tab:transfer} contains the calibration results for different implementations of outcome-conditional recalibration. When recalibrating for a fixed number of partitions, e.g. $G = 5$, the calibration error deteriorates when evaluated on a finer partition, though outcome-conditional recalibration still outperforms all baselines. In contrast, training using a finer partition of the outcome space also performs well on coarser partitions. When interest is on calibration for many different partitions of the outcome space, an iterated recalibration method performs well; this approach is described in the following subsection. However, this is not much more beneficial than the more parsimonious outcome-conditional recalibration approach covered in the main text.

\begin{table}[htbp]
\centering\scriptsize
\setlength{\tabcolsep}{4pt}
\caption{Evaluation of recalibration schemes when calibration is assessed on different partitions to those used for training the outcome-conditional recalibration scheme. The recalibration is trained using $G = 5, 16, 20$ partitions, all defined using evenly spaced quantiles of the outcome distribution. Calibration error is then calculated for the following partitions: ``Cal-all'' is the standard (partition-free) calibration error; $G = 5, 10, 20$ evenly spaced partitions; $G=5$ shift sets quantile breaks at $[0.1, 0.3, 0.5, 0.7]$; \texttt{union16} pairs adjacent atoms of the $16$-cell partition; \texttt{tail} is a partition at quantile breakpoints $0.5$, $0.9$, $0.95$ and $0.99$ quantiles; \texttt{rand} averages three partitions with 4 break levels drawn uniformly per seed; \emph{mean} averages the seven partition columns. Results are aggregated over all datasets with five random seeds, with Bayesian ridge as the base model. All calibration errors have been scaled by $10^3$ to aid readability.\\
 }
\label{tab:transfer}
\begin{tabular}{lrrrrrrrrr}
\toprule
Method                              & Cal-all & $G{=}5$ & \texttt{union16} & $G{=}5$ shift & $G{=}10$ & $G{=}20$ & \texttt{tail} & \texttt{rand} & mean \\
\midrule
Uncalibrated                        & $11.8$ & $198.4$ & $178.4$ & $235.1$ & $158.9$ & $133.8$ & $252.4$ & $217.3$ & $196.3$ \\
KFE18                               & $2.2$ & $177.7$ & $163.9$ & $218.0$ & $147.5$ & $124.5$ & $225.9$ & $195.5$ & $179.0$ \\
KD22                                & $4.9$ & $20.7$ & $37.8$ & $36.6$ & $44.7$ & $71.7$ & $\mathbf{118.7}$ & $48.0$ & $54.0$ \\
Ours: OC $G{=}5$              & $7.9$ & $\mathbf{8.7}$ & $32.0$ & $31.3$ & $41.5$ & $80.2$ & $154.7$ & $42.3$ & $55.8$ \\
Ours: OC, $G{=}16$             & $10.1$ & $13.2$ & $\mathbf{19.0}$ & $\mathbf{15.4}$ & $\mathbf{26.3}$ & $56.0$ & $132.8$ & $37.9$ & $\mathbf{42.9}$ \\
Ours: OC, $G{=}20$             & $10.0$ & $13.2$ & $20.2$ & $16.2$ & $26.9$ & $\mathbf{54.1}$ & $132.7$ & $40.3$ & $43.4$ \\
Ours: OC + QR, $G{=}5$            & $1.6$ & $15.0$ & $27.6$ & $36.3$ & $44.4$ & $83.7$ & $188.4$ & $48.4$ & $63.4$ \\
Ours: OC + QR, $G{=}16$           & $\mathbf{1.6}$ & $14.7$ & $19.9$ & $20.4$ & $29.6$ & $56.7$ & $148.7$ & $39.2$ & $47.0$ \\
Ours: OC + QR, $G{=}20$           & $1.7$ & $15.3$ & $20.4$ & $20.3$ & $30.2$ & $56.0$ & $135.7$ & $42.0$ & $45.7$ \\
Iterated recalibration, $10$ rounds & $4.9$ & $13.5$ & $19.9$ & $20.6$ & $31.1$ & $56.7$ & $122.7$ & $\mathbf{37.1}$ & $43.1$ \\
\bottomrule
\end{tabular}
\end{table}

\subsection{Iterated recalibration}

Any single application of outcome-conditional recalibration is in-sample calibrated on only one family of subsets (the family on which it is trained). Inspired by the multicalibration strategy from \citet{hebert-johnson_multicalibration_2018}, we now introduce a strategy to recalibrate predictive distributions on many subsets of the outcome space. This approach applies outcome-conditional recalibration iteratively on different partitions, and composes the resulting transformations. 

\textbf{Setup.} Consider $n_b$ break points
$b_0 = -\infty < b_1 < \cdots < b_{n_b-1} = +\infty$, possibly at
equally spaced quantile levels of the outcome space, then define the set of all subsets defined by these breakpoints,
\[
    \mathcal{S} \;=\; \bigl\{[b_j, b_k) : 0 \le j < k \le n_b-1\bigr\},
    \qquad |\mathcal{S}| = \tbinom{n_b}{2}.
\]
For example, $n_b = 17$ yields $|\mathcal{S}| = 136$ subsets, including $\R$ itself. 

\textbf{Method.} Outcome-conditional recalibration is applied in rounds. In each round, it finds the subset in $\mathcal{S}$ with the largest calibration error. It then recalibrates this subset, whilst leaving the predictive distribution on the complement of this subset untouched (up to a rescaling to ensure a valid distribution function). This is then repeated for $R$ rounds, after which the resulting predictive distribution is a composition of the one-subset outcome-conditional recalibrations. The final transformation $G$ produces the iterated recalibrated prediction $\Ftilde_i(y) = G_i(F_i(y))$. Each round of the algorithm is described in more detail below:

\paragraph{For rounds $1, \dots,R,$}
\begin{enumerate}
\item Find the subset $S = [b_j, b_k)\in \mathcal{S}$ with the largest subset calibration error (calibration error over $1.65/n_S$). 
\item Compute current break point probabilities: $A_i = G_i(F_i(b_j))$, $B_i = G_i(F_i(b_k))$, $\Delta_i = B_i - A_i$ as well as CPIT-values: $v_i = (G_i(F_i(y_i)) - A_i)/\Delta_i$ for the observations $y_i \in S$. For round 1, $G_i$ is the identity function.
\item Apply outcome-conditional recalibration on $S = [b_j, b_k)$: %and $S^C = (-\infty, b_j] \cup (b_k, \infty)$, whilst setting $T^{S^C} = \text{id}.$:
\begin{itemize}
    \item \emph{Shape}: fit $T^S$ using isotonic regression of $\{v_i\}$ on its own
      normalised ranks (see Appendix~\ref{app:sample}).
      \item \emph{Mass}: writing $c = \log(c^S/c^{S^c})$, the recalibrated mass assigned to $S$ simplifies to,
      \[
        m_i =\frac{c^S \Delta_i}{c^S \Delta_i + c^{S^c}(1 - \Delta_i)} = \mathrm{logit}^{-1}\bigl(\mathrm{logit}(\Delta_i) + c\bigr).
      \]
      The recalibration equation $\sum_i \mathrm{logit}^{-1}(\mathrm{logit}(\Delta_i) + c)/N = q_S$ can be solved by bisection. 
\end{itemize}
\item Update $G_i$ as: 
      \[
        G_i(u) \leftarrow
        \begin{cases}
          s_i G_i(u), & u \le F_i(b_j),\\[2pt]
          A_i' + m_i\,T^S\bigl((G_i(u)-A_i)/\Delta_i\bigr), & F_i(b_j)<u<F_i(b_k),\\[2pt]
          1 - s_i\bigl(1-G_i(u)\bigr), & u \ge F_i(b_k),
        \end{cases}
      \]
      where $s_i = (1 - m_i)/(1 - \Delta_i)$ and $A_i' = s_i A_i$. This is~\eqref{eq:tail_recal} for the partition $\{S, S^c\}$ with $T^{S^c} = \mathrm{id}$. 
\end{enumerate}
We find that the number of rounds is usually of minor importance, and running for a fixed 10 rounds yields good performance in the application above. As single-step outcome-conditional recalibration is extremely fast, this is very feasible computationally. It is also possible to perform early stopping on a held-out dataset, but we find that this does not increase performance much.

\end{document}